\documentclass[runningheads]{llncs}
\usepackage[left=3.5cm,right=2.5cm,top=2.5cm,bottom=2.5cm]{geometry}  

\usepackage{amsmath}
\usepackage{color}      
\usepackage{subcaption}
\usepackage{amsfonts}
\usepackage{hyperref}   
\usepackage{tikz-cd}
\usepackage{tikz}
\usetikzlibrary{fit,positioning,arrows,matrix,trees,shapes, shapes.multipart}
\usepackage{blkarray}
\usepackage{mathrsfs}
\usepackage{graphicx}
\usepackage{lmodern}
\usepackage{mathtools}

\DeclareMathOperator*{\argmax}{arg\,max}

\begin{document}
\title{Solving Tree Problems with Category Theory}
%
\titlerunning{Solving Tree Problems with Category Theory}
%
\author{Rafik Hadfi}
\authorrunning{R. Hadfi}
\institute{School of Psychological Sciences,\\ Faculty of Medicine Nursing and Health Sciences,\\ Monash University, Australia\\\email{rafik.hadfi@monash.edu}}
\maketitle
\begin{abstract} 
Artificial Intelligence (AI) has long pursued models, theories, and techniques to imbue machines with human-like general intelligence. Yet even the currently predominant data-driven approaches in AI seem to be lacking humans' unique ability to solve wide ranges of problems. This situation begs the question of the existence of principles that underlie general problem-solving capabilities. We approach this question through the mathematical formulation of analogies across different problems and solutions. We focus in particular on problems that could be represented as tree-like structures. Most importantly, we adopt a category-theoretic approach in formalising tree problems as categories, and in proving the existence of equivalences across apparently unrelated problem domains. We prove the existence of a functor between the category of tree problems and the category of solutions. We also provide a weaker version of the functor by quantifying equivalences of problem categories using a metric on tree problems.
\end{abstract}

\section{Introduction}

General problem-solving has long been one of main goals of Artificial Intelligence (AI) since the early days of Computer Science. Many theories on generality and problem-solving have been proposed and yet the task of building machines that could achieve human-level intelligence is still in its infancy.

Humans are good at solving problems because they can reason about unknown situations. They are capable of asking hypothetical questions that can effectively be answered through analogical reasoning. Analogical reasoning is when concepts from one space are mapped to the concepts of another space after noticing structural similarities or equivalences between the two. For instance, having observed how a clay vase is being moulded, one could learn to mentally manipulate other clay objects. Similarly, learning to solve one puzzle could be accelerated if one could relate to previously mastered puzzle games.

Solving problems using analogies requires the ability to identify relationships amongst complex objects and transform new objects accordingly. In its canonical form, an analogy is usually described as <<A is to B as C is to D>>. Despite their intuitive appeal, analogies do have the drawback that, if the structure is not shared across the full problem space, we might end up with a distorted understanding of a new problem than if we had not tried to think analogically about it. It is therefore crucial to find a formalism that translates problems into the representation that allows comparisons and transformations on its structures.

Category Theory is a powerful mathematical language capable of expressing equivalences of structures and analogies. It was introduced in 1942--45 by Saunders MacLane and Samuel Eilenberg as part of their work on algebraic topology \cite{eilenberg1945general}. What seemed to be an abstract theory that had no content turned out to be a very flexible and powerful language. The theory has become indispensable in many areas of mathematics, such as algebraic geometry, representation theory, topology, and many others. Category Theory has also been used in modelling the semantics of cognitive neural systems \cite{healy2000category}, in describing certain aspects of cognition such as systematicity \cite{phillips2010categorial,phillips2016systematicity}, in formalising artificial perception and cognition \cite{arzi1999perceive,magnan1994category}, and in advancing our understanding of brain function \cite{ramirez2010new} and human consciousness \cite{tsuchiya2016using}.

In the present work, we propose a category-theoretic formalism for a class of problems represented as arborescences \cite{gordon1989greedoid}. We strongly think that many decision-making and knowledge representation problems are amenable to such structures \cite{diuk2013divide,rasmussen2014hierarchical}. The category-theoretic approach to general problem-solving comes as a qualitative alternative to the currently dominant quantitative, data-driven approaches that rely on Machine Learning and Data Science. We aim at identifying the types or common classes in tree problems using category equivalences. The number of types should be much smaller compared to what data-driven approaches to problem-solving usually yield. It should be easier to identify a new situation by its own type and apply the right transformations to obtain the desired solution. Such transformations will be formalised using functors and aim at computing the solutions to the tree problem in multiple ways.

The main contributions of the paper are twofold. We formalise some the most common problems in AI literature in the most generic way possible and give them an algebraic structure suitable to category theory and its functor-based formulation of analogies. The second contribution is the way we combine the problems and their solutions into two distinct categories, allowing us to define equivalence classes on problems regardless of the existence of solutions.

The paper is structured as following. In the next section, we review some of the previous work on general problem-solving and the usages of analogy. In section 3, we introduce the class of problems we are interested in. In section 4, we show how to translate such problems to a category-theoretic representation. In section 5, we show how solutions could be formalised based on functors and category equivalences. Finally, we conclude and highlight the future directions.
 
\section{Related Work}

General problem-solving is not new in Artificial Intelligence and many authors have proposed guidelines for this line of search \cite{laird2010cognitive,rosa2016framework}. One of the earliest theories of general problem-solving was proposed in \cite{newell1959report} and relied on recursive decompositions of large goals into subgoals while separating problem content from solution strategies. The approach became later known as the cognitive architecture SOAR \cite{laird2012soar} and is amongst the first attempts to a unified theory of cognition \cite{newell1992unified}. In the context of universal intelligence, \cite{hutter2004universal} proposed a general theory that combines Solomonoff induction with sequential decision theory, and was implemented as a reinforcement learning agent called AIXI. The downside is that AIXI is incomputable and relies on approximations \cite{veness2010monte}. Other approaches to generalised intelligence rely on transferring skills or knowledge across problem domains \cite{sharma2007transfer,taylor2009transfer}. For instance, \cite{bonet2009solving,bonet2015policies} focuses on partially observable non-deterministic problems (PONDP) and provides a way of transferring a policy from a PONDP to another one with the same structure

The ability to generalise across different situations has long been the hallmark of analogy-making. One the first attempts to formalise analogies was through the concept of elementary equivalence in logical Model Theory \cite{keisler1990model}. Most recently, deep convolutional neural networks (CNN) have enabled us to solve visual analogies by transforming a query image according to an example pair of related images \cite{liao2017visual,reed2015deep}. The approach does not exploit the regularities between the transformations and seems to follow one particular directed path in the commutative diagram of the problem if expressed in category-theoretic terms. 

As mentioned in the introduction, Category Theory constitutes an elegant framework that can help conceptualise the essence of general problem-solving, and abstract how the different paradigms of AI implement the solutions algorithmically. The practical component of the theory is that it can redefine the algorithms in terms of functors (or natural transformations) across problem and solution categories. However, we think that the real challenge resides in the ability to implement the type of functors that can systematically map input (problem) to output (solution) in a manner similar to what is done in Machine Learning. Although the category-theoretic approach to general problem-solving is still at an early stage of development, the work of \cite{izbicki2013algebraic} can be considered as a recipe for a scalable and systematic usage of functors, albeit in the area of Machine Learning. Particularly, the author defines a training algorithm as a monoid homomorphism from a free monoid representing the data set, to a monoid representing the model we want to train \cite{izbicki2013two}. Most instances of such ``homomorphic trainer'' type class are related to statistics or Machine Learning, but the class is much more general than that, and could for instance be used to approximate NP-complete problems \cite{izbicki2013two}. This approach is shown to improve the learning scalability in the sense that it starts by learning the problem independently on small subsets of the data before merging the solutions together within one single round of communication.

The more general framework of \cite{arjonilla2017general} lays the foundation of a formal description of general intelligence. This framework is based on the claim that cognitive systems learn and solve problems by trial and error \cite{arjonilla2015three}. The authors introduce cognitive categories, which are categories with exactly one morphism between any two objects. The objects of the categories are interpreted as states and morphisms as transformations between those states. Cognitive problems are reduced to the specification of two objects in a cognitive category: the current state of the system and the desired state. Cognitive systems transform the target system by means of generators and evaluators. Generators realise cognitive operations over a system by grouping morphisms, while evaluators group objects as a way to generalise current and desired states to partially defined states. 

For our approach to general problem-solving to work, an agent should not only be capable of solving the problems specific to its native ecological niche, but should also be capable of transcending its current conceptual framework and manipulate the class of the problems itself. This would allow the agent to solve new problems once deployed in new contexts that share some equivalences with the previously encountered contexts. Generalising across different contexts could be achieved for instance using natural transformations mapping functors between known categories of problems and solutions to new ones.

The capacity of the agent to represent and manipulate common structural relationships across equivalently cognizable problem domains is known in cognitive sciences as systematicity \cite{fodor1988connectionism}. In general, it is an equivalence relation over cognitive capacities, a kind of generalisation over cognitive abilities. The problem with systematicity is that it fails in explaining why cognition is organised into particular groups of cognitive capacities. The author in \cite{phillips2017general} hypothesises that the failures of systematicity arise from a cost/benefit trade-off associated with employing one particular universal construction. A universal construction is defined as the necessary and sufficient conditions relating collections of mathematically structured objects. Most importantly, the author proposes adjunction as universal construction for trading the costs and benefits that come from the interaction of a cognitive system with their environment, and where general intelligence involves the effective exploitation of this trade-off.

One distinction between our approach and that of \cite{phillips2017general} is that we not consider the interaction between the agent and the environment for which the adjunction is defined. We only focus on the functor mapping problems to solutions and do not define its adjoint functor. For our goal of general-problem solving, and given the way we define the problem and solution categories, it would not make much sense to look for a problem given its solutions.


\section{Tree Problems}

\subsection {Definition}
We define tree problems as an umbrella term for a class of problems in the area of problem-solving in general and in combinatorial optimisation in particular. While tree problems may be formulated in a number of ways, they all require a rooted arborescent interconnection of objects and an objective function. Given a directed rooted tree with predefined edge labels and a set of terminal vertices, the corresponding tree problem possesses at most one solution. The solution corresponds to a path from the root of the tree to one of its terminal nodes. A problem $\mathcal{P}$ is formally represented by the tuple $T_\mathcal{P}=(T,\mathcal{L}, \mathcal{A})$, defined as following.

\begin{itemize} 
    \item The tuple $T=(r, V, E)$ is a labelled tree with root $r$, a set of nodes $V$, and a set of edges $E \subseteq V \times V$. The set $V$ is partitioned into a set of internal nodes $I$ and a set of terminal nodes $\Omega$. We note $V(T)$ and $E(T)$ as shorthands for the vertices and edges of the tree $T$.
    \item  The tuple $\mathcal{L}=(\mathcal{L}_V, \mathcal{L}_E)$ defines the ``labelling'' functions $\mathcal{L}_V: V \mapsto \mathbb{R}^n$ and $\mathcal{L}_E: E \mapsto \mathbb{R}^m$. The numbers $n$ and $m$ are respectively the numbers of vertice and edge features.
    \item The algorithm $\mathcal{A}: T \mapsto S_\mathcal{P}$ implements an objective function that assigns solution $S_\mathcal{P}$ to $T$.
\end{itemize}

Such tree-based formalism is meant to encode a number of decision problems in the most generic fashion. Such problems could share the same structure as it is defined by the tree and differ only in the labels or features that are assigned to the nodes and edges. In the following, we choose to reduce the space of tree structures and restrict our problems to problems that could be represented as binary trees. It is in fact possible to translate $n$-ary representations to binary representations by transforming branchings like \includegraphics[height=0.25cm]{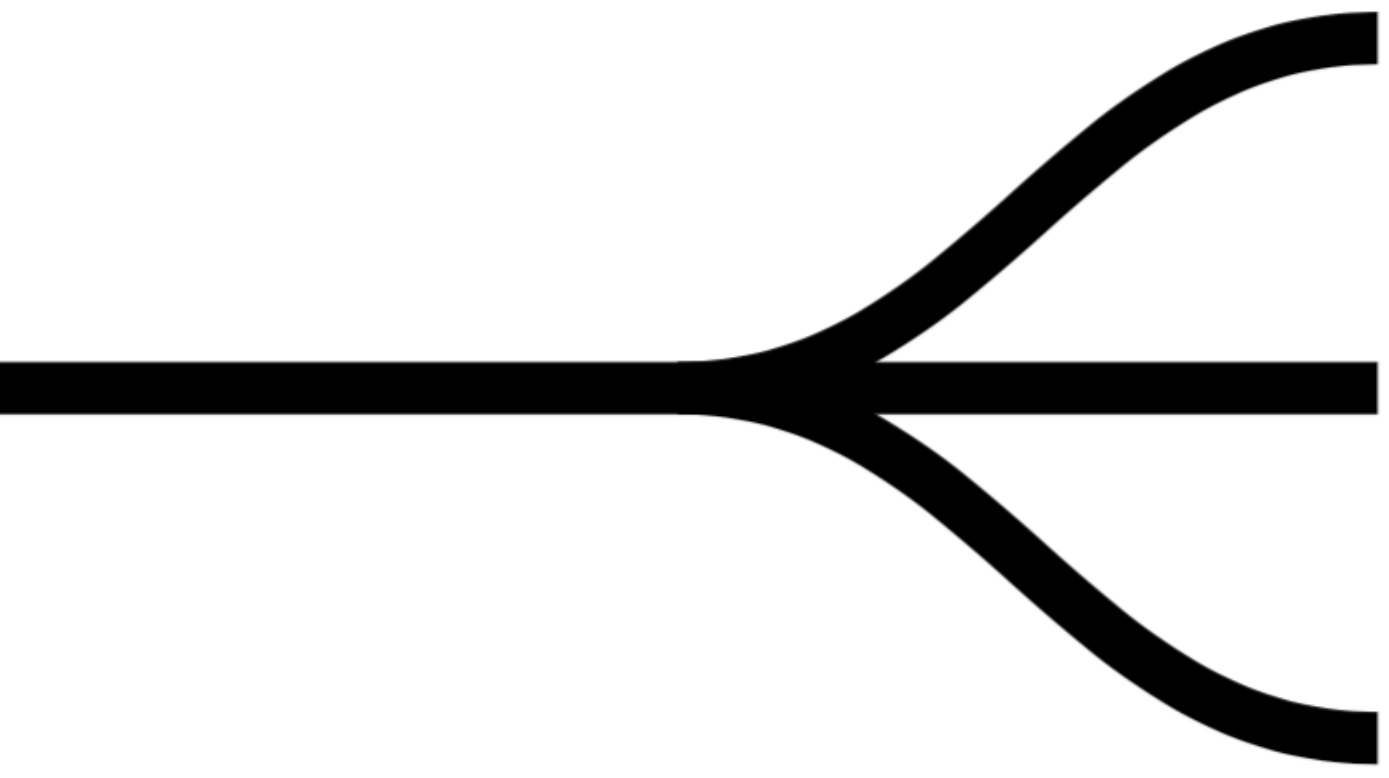} into \includegraphics[height=0.25cm]{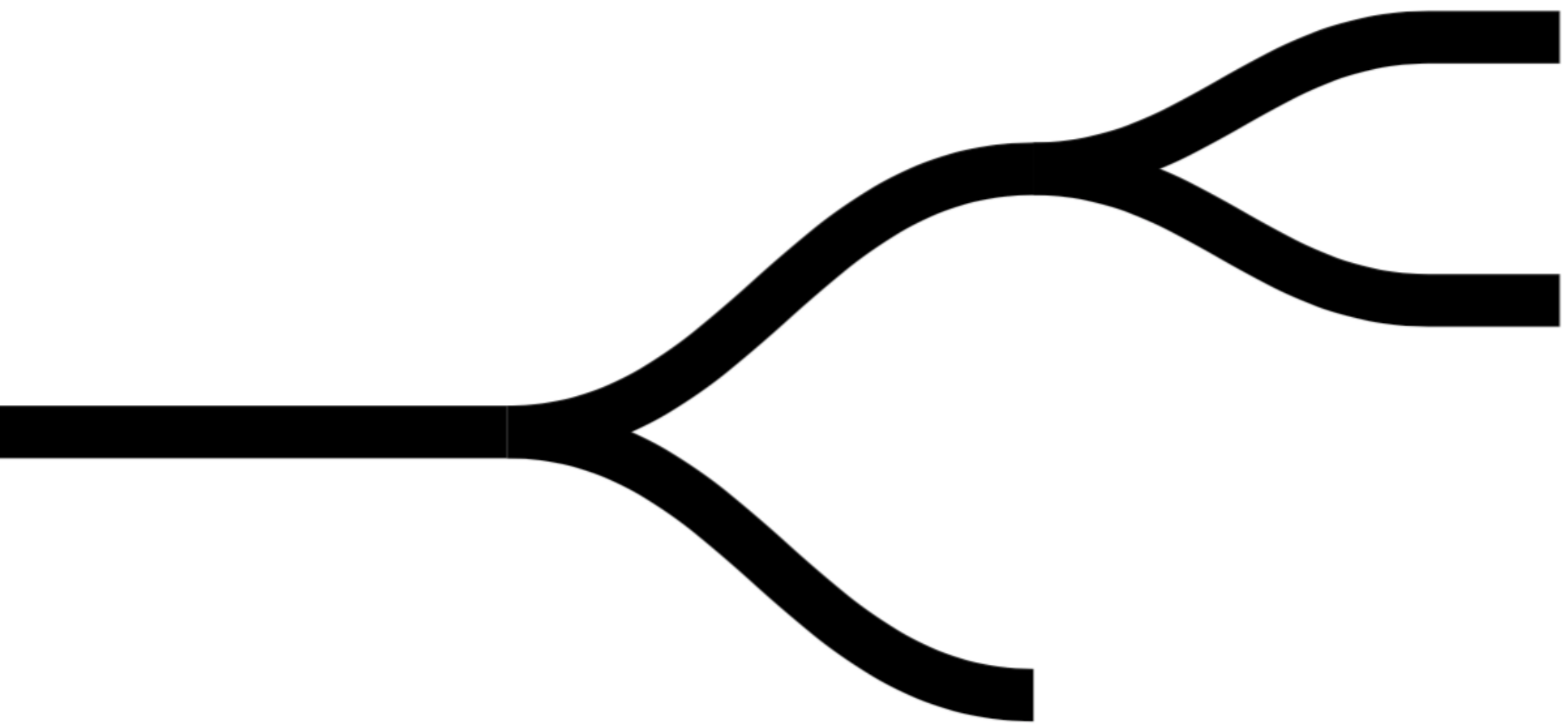} and altering the edge lengths.

In the following, a solution $S_{\mathcal{P}}$ to problem $\mathcal{P}$ will be encoded as a binary vector of the form $S_{\mathcal{P}} \in \{0,1\}^n$. That is, $S_{\mathcal{P}}$ assigns $1$ to its $i^{th}$ entry if edge $e_i$ is in the solution path. Note that it is possible to imagine solutions that do not possess any problem, but we do not addess such cases.

\subsection{Characteristic matrix of a tree problem}
\label{charaMsection}
To find a canonical characterisation of a tree problem we start by defining $\mathbb{T}$ as the set of all rooted trees with $k$ terminal nodes ($|\Omega|=k$). We say that trees $T_a,T_b \in \mathbb{T}$ have the same labeled shape or topology if the set of all the partitions of $\Omega$ admitted by the internal edges of $T_a$ is identical to that of $T_b$, and we write $T_a \simeq T_b$. We say that $T_a = T_b$ if they have the same topology and the same labelling: $\mathcal{L}_{E(T_a)} = \mathcal{L}_{E(T_b)}$. 
For any $T \in \mathbb{T}$, we define $\mu^1_{i,j}$ as the number of edges on the path from the root to the most recent common ancestor of terminal nodes $i$ and $j$ and $\mu^{\ell}_{i,j}$ as the $\ell^{th}$ feature of this edge, and set $p^\ell_{i}$ as the value of the $\ell^{th}$ feature corresponding to the pendant edge to tip $i$. Given all pairs of terminal nodes $\Omega$, we define the characteristic matrix of $T$ as in (\ref{Mmatrix}).
\begin{align}
M(T) = 
\begin{pmatrix}
\mu^{1}_{1,2} & \mu^{1}_{1,3} & \ldots & \mu^1_{i,j} & \ldots & \mu^{1}_{k-1, k} & 1 & \ldots & 1 \\ 
\mu^{2}_{1,2} & \mu^{2}_{1,3} & \ldots & \mu^2_{i,j} & \ldots & \mu^{2}_{k-1, k} & p^{2}_1 & \ldots & p^{2}_k \\
\vdots  & \vdots  & \ddots & \vdots &  & \vdots  &&\ddots & \vdots \\ 
\mu^{\ell}_{1,2} & \mu^{\ell}_{1,3} & \ldots & \mu^{\ell}_{i,j} & \ldots & \mu^{\ell}_{k-1, k} & p^{\ell}_1 & \ldots & p^{\ell}_k \\
\vdots  & \vdots  & \ddots & \vdots &  & \vdots && \ddots& \vdots \\ 
\mu^{m+1}_{1,2} & \mu^{m+1}_{1,3} &\ldots & \mu^{m+1}_{i,j} & \ldots & \mu^{m+1}_{k-1, k} & p^{m+1}_1 & \ldots & p^{m+1}_k \\
 \end{pmatrix}_{m+1\ \times\ \binom{k}{2} + k}
 \label{Mmatrix}
\end{align}

The first row of (\ref{Mmatrix}) captures the tree topology and the other rows capture both the topology and the $m$ features encoded by $\mathcal{L}_E: E \mapsto \mathbb{R}^m$. The feature vectors $M_{j \ge 2}$ are in fact inspired from the vectors of cophenetic values \cite{cardona2013cophenetic}. Note that we have $m+1 < \binom{k}{2} + k$ since the number of edges of the tree usually exceeds the number of features that characterise most basic tree problems. For instance, such features are usually restricted to topology, length, probability, or cost. We finally take the convex combination of the vectors to obtain the characteristic function (\ref{cvx}).
\begin {align}
\phi_{\lambda}(T) = \lambda M 
\label{cvx}
\end {align}
The characteristic form $\phi_\lambda$(T) is parameterised by $\lambda \in [0,1]^{m+1}$ with $\sum_{j=1}^{m+1} \lambda_j = 1$. The elements of $\lambda$ specify the extent to which different tree features contribute in characterising the tree $T$. In this sense, one feature may dominate other features as the elements of $\lambda$ increase from 0 to 1. 
\begin{figure}[h]
\centering
\begin{subfigure}[b]{0.4\columnwidth}
               \begin{tikzpicture}[level distance=3.5cm,
                      level 1/.style={sibling distance=2.5cm},
                      level 2/.style={sibling distance=1.5cm},
                      xcirc/.style={	circle,draw=black,
                      			fill=black, inner sep=0pt,minimum size=6pt},
                      grow=right, scale=0.7]
                      \node  {r}
                        child {node  {$\beta$} 
                          	child {node {$a$} edge from parent  node[below] {$1.2$}}
                         	child {node  {$b$} edge from parent  node[above] {$3.1$}}
            	       edge from parent  node[below] {$0.4$} 
                        }  
                        child {node {$\alpha$}  
                            	child {node {$c$} edge from parent node[below]  {$5.7$} }
                             child {node {$d$}	edge from parent  node[above] {$1$} }
            		edge from parent node[above] {$2.5$}          	
            	  };
                    \end{tikzpicture}
                \caption{Structure and features of $T$}
                \label{treetopinf}
            \end{subfigure}
            \begin{subfigure}[b]{0.3\columnwidth}
            \begin {align}
               M(T) =\ \ \begin{blockarray}{cccccccccc}
                ab & ac & ad & bc & bd & cd & p_a & p_b & p_c & p_d \\
                \begin{block}{(cccccccccc)}
                  1 & 0 & 0 & 0 & 0 & 1 &  1 & 1 & 1 & 1\\
                  0.4 & 0 & 0 & 0 & 0 &  2.5 & 1.2 & 3.1 & 5.7 & 1\\
                \end{block}
                \end{blockarray} \nonumber
            \end {align}
            \caption{Characteristic matrix of $T$}
            \label{chmatrix}
        \end{subfigure}
\caption{Tree problem and its matrix representation}
\label{mazetrees}
\end{figure}
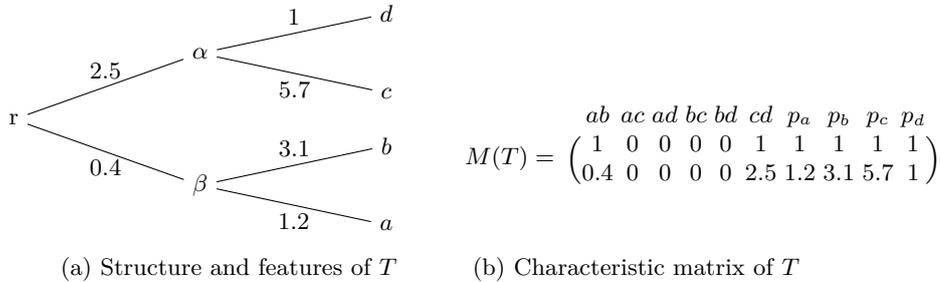

For instance, the tree in figure \ref{treetopinf} is characterised by its topology and one feature corresponding to the length of its branches. The corresponding characteristic matrix $M(T)$ is given in figure \ref{chmatrix}. Note that the matrix $M(T)$ is constructed from the mappings defined by $\mathcal{L}_E(T)$. For instance, the first column of $M(T)$ is in fact $\mathcal{L}_{E(T)}((r, \beta)) = (1, 0.4)$.

\subsection {Instances of tree problems}
It is possible to find many instances of problems in AI that are reducible to tree structures. For instance, simply connected mazes are mazes that contain no loops or disconnected components. Such mazes are equivalent to a rooted tree in the sense that if one pulled and stretched out the paths in the maze in a continuous way, the result could be made to resemble a tree \cite{Maze2Tree}. Mathematically, the existence of a continuous deformation between the maze and a rooted tree means that they are homeomorphic. For two spaces to be homeomorphic we only need a continuous mapping with a continuous inverse function. A homeomorphism or topological isomorphism is a continuous function between topological spaces that has a continuous inverse function. The existence of such mapping is what will be exploited in our approach by moving between problems' space and solutions' space in a well-principled manner. Instead of working directly on complex structures like mazes, one could convert them to trees, and then study the existence of homomorphisms \cite{csikvari2013graph} and other transformations.

For example, figure \ref{mazetree} illustrates a maze search problem and its homologous decision tree problem. In figure \ref{fig:sub:Maze}, $\mathcal{L}_E: E \mapsto \mathbb{R}$ assigns lengths to the edges of the tree $T$, and $\mathcal{L}_V: \Omega \mapsto \mathbb{R}$ assigns outcomes to the terminal nodes $\Omega$ comprised of goal node(s) and dead-ends.

\begin{figure}[h]
\centering
\begin{subfigure}{0.4\textwidth}
\begin{tikzpicture}[scale=0.33, path/.style={black, thick, <-, >=stealth,rounded corners=0.2cm, shorten >=2pt}]
\draw[step=20cm, gray, very thin] (0, -10) grid (16, 0);
\fill[black] (0, 0) rectangle (1, -1);
\fill[black] (1, 0) rectangle (2, -1);
\fill[black] (2, 0) rectangle (3, -1);
\fill[black] (3, 0) rectangle (4, -1);
\fill[black] (4, 0) rectangle (5, -1);
\fill[black] (5, 0) rectangle (6, -1);
\fill[black] (6, 0) rectangle (7, -1);
\fill[black] (7, 0) rectangle (8, -1);
\fill[black] (8, 0) rectangle (9, -1);
\fill[black] (9, 0) rectangle (10, -1);
\fill[black] (10, 0) rectangle (11, -1);
\fill[black] (11, 0) rectangle (12, -1);
\fill[black] (12, 0) rectangle (13, -1);
\fill[black] (13, 0) rectangle (14, -1);
\fill[black] (14, 0) rectangle (15, -1);
\fill[black] (15, 0) rectangle (16, -1);
\fill[black] (0, -1) rectangle (1, -2);
\fill[black] (0, -2) rectangle (1, -3);
\fill[black] (2, -2) rectangle (3, -3);
\fill[black] (3, -2) rectangle (4, -3);
\fill[black] (4, -2) rectangle (5, -3);
\fill[black] (5, -2) rectangle (6, -3);
\fill[black] (6, -2) rectangle (7, -3);
\fill[black] (7, -2) rectangle (8, -3);
\fill[black] (8, -2) rectangle (9, -3);
\fill[black] (9, -2) rectangle (10, -3);
\fill[black] (10, -2) rectangle (11, -3);
\fill[black] (0, -3) rectangle (1, -4);
\fill[black] (2, -3) rectangle (3, -4);
\fill[black] (10, -3) rectangle (11, -4);
\fill[black] (12, -3) rectangle (13, -4);
\fill[black] (0, -4) rectangle (1, -5);
\fill[black] (2, -4) rectangle (3, -5);
\fill[black] (4, -4) rectangle (5, -5);
\fill[black] (5, -4) rectangle (6, -5);
\fill[black] (6, -4) rectangle (7, -5);
\fill[black] (7, -4) rectangle (8, -5);
\fill[black] (8, -4) rectangle (9, -5);
\fill[black] (10, -4) rectangle (11, -5);
\fill[black] (12, -4) rectangle (13, -5);
\fill[black] (13, -3) rectangle (14, -4);
\fill[black] (15, -4) rectangle (16, -5);
\fill[black] (0, -5) rectangle (1, -6);
\fill[black] (2, -5) rectangle (3, -6);
\fill[black] (8, -5) rectangle (9, -6);
\fill[black] (10, -5) rectangle (11, -6);
\fill[black] (0, -6) rectangle (1, -7);
\fill[black] (2, -6) rectangle (3, -7);
\fill[black] (6, -6) rectangle (7, -7);
\fill[black] (7, -6) rectangle (8, -7);
\fill[black] (8, -6) rectangle (9, -7);
\fill[black] (12, -6) rectangle (13, -7);
\fill[black] (13, -6) rectangle (14, -7);
\fill[black] (15, -6) rectangle (16, -7);
\fill[black] (0, -7) rectangle (1, -8);
\fill[black] (2, -7) rectangle (3, -8);
\fill[black] (6, -7) rectangle (7, -8);
\fill[black] (0, -8) rectangle (1, -9);
\fill[black] (2, -8) rectangle (3, -9);
\fill[black] (4, -8) rectangle (5, -9);
\fill[black] (10, -8) rectangle (11, -9);
\fill[black] (0, -9) rectangle (1, -10);
\fill[black] (6, -9) rectangle (7, -10);
\fill[black] (8, -9) rectangle (9, -10);
\fill[black] (14, -9) rectangle (15, -10);
\fill[black] (15, -9) rectangle (16, -10);
\fill[black] (15, -8) rectangle (16, -9);
\fill[black] (15, -7) rectangle (16, -8);
\fill[black] (15, -6) rectangle (16, -7);
\fill[black] (15, -5) rectangle (16, -6);
\fill[black] (15, -4) rectangle (16, -5);
\fill[black] (14, -6) rectangle (15, -7);
\fill[black] (12, -5) rectangle (13, -6);
\fill[black] (15, -3) rectangle (16, -4);
\fill[black] (15, -2) rectangle (16, -3);
\fill[black] (15, -1) rectangle (16, -2);
\fill[black] (13, -9) rectangle (14, -10);
\fill[black] (12, -9) rectangle (13, -10);
\fill[black] (11, -9) rectangle (12, -10);
\fill[black] (10, -9) rectangle (11, -10);
\fill[black] (9, -9) rectangle (10, -10);
\fill[black] (8, -9) rectangle (9, -10);
\fill[black] (7, -9) rectangle (8, -10);
\fill[black] (6, -9) rectangle (7, -10);
\fill[black] (5, -9) rectangle (6, -10);
\fill[black] (4, -9) rectangle (5, -10);
\fill[black] (3, -9) rectangle (4, -10);
\fill[black] (2, -9) rectangle (3, -10);
\fill[black] (1, -9) rectangle (2, -10);
\fill[black] (1, -9) rectangle (2, -10);
\fill[black] (4, -6) rectangle (5, -7);
\fill[black] (4, -7) rectangle (5, -8);
\fill[black] (6, -8) rectangle (7, -9);

\fill[black!50!white] (13.5, -8.5) circle (0.3cm);
\fill[black!50!white] (1.5, -6.5) circle (0.3cm);
    
    \draw[above right] (13.5, -8.5) node(s){$r$};
    \draw[below] (1.5, -6.6) node(s){$\omega$};
    \draw[above] (13.2, -2.6) node(s){$\ell$};
    
    \draw [path] (3.5, -8.5) -- (3.5, -3.5) -- (4.5, -3.5) -- (5.5, -3.5) -- (6.5, -3.5) -- (7.5, -3.5) -- (8.5, -3.5) 
    		-- (9.5, -3.5) -- (9.5, -4.5) -- (9.5, -5.5) -- (9.5, -6.5) -- (9.5, -7.5) -- (13.5, -7.5) -- (13.5, -8.5) ;
    \draw [path] (1.5, -5.5) -- (1.5, -1.5) -- (11.5, -1.5) -- (11.5, -7.5) -- (13.5, -7.5) -- (13.5, -8.5) ;
    \draw [path]  (13.5, -4.5) --  (13.5, -5.5) --  (14.5, -5.5) --  (14.5, -2.5) -- (11.5, -2.5) -- (11.5, -7.5) -- (13.5, -7.5) -- (13.5, -8.5) ;
    \draw [path]   (9.5, -8.5) -- (8.5, -8.5) --  (7.5, -8.5) --  (7.5, -7.5) -- (10.5, -7.5) ;
    \draw [path]  (7.5, -5.5)-- (3.4, -5.5) ;
    \draw [path]  (5.5, -8.5) -- (5.5, -5.4) ;

\end{tikzpicture}

    \caption{Maze Search Problem}
    \label{fig:sub:Maze}
\end{subfigure}
\begin{subfigure}{0.49\textwidth}
        \begin{tikzpicture}[level distance=2.5cm,
          level 1/.style={sibling distance=2.7cm},
          level 2/.style={sibling distance=1.5cm},
          xcirc/.style={circle,draw=black, fill=white, inner sep=0pt,minimum size=6pt},
          grow=right, scale=0.7]
          \node {r}
            child {node [xcirc] {} 
              	child {node {$\omega_1$}
	edge from parent  node[below] {$1-p_3$}
	}
             	child {node [xcirc] {} 
                	 	 child {node[xcirc]  {} 
            		  	child {node {$\omega_2$}}       
            		 	child {node {$\omega_3$} edge from parent  node[above] {$p_5$}}
        	  			}      
            		child {node {$\omega_4$} edge from parent  node[above] {$p_4$}}   }
            }  
            	child {node [xcirc] {}  
                		child {node {$\omega_5$}}
                 	child {node {$\omega_6$}	 edge from parent  node[above] {$p_2$} }
		 	edge from parent  node[above] {$p_1$}          
            		};
        \end{tikzpicture}
            \caption{Decision Tree Problem}
            \label{fig:sub:Dtree}
        \end{subfigure}

\caption{Two Homeomorphic Problems}
\label{mazetree}
\end{figure}
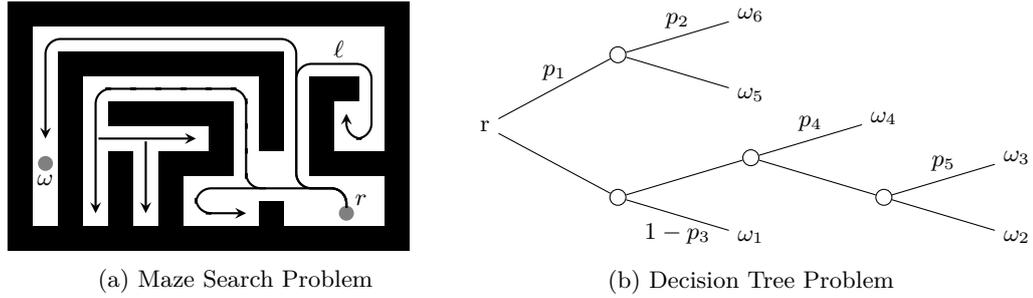

In the decision tree of figure \ref{fig:sub:Dtree}, $\mathcal{L}_V: \Omega \mapsto \mathbb{R}$ maps terminal nodes to outcomes and $\mathcal{L}_E: E \mapsto [0,1]$ maps sub-branches to the probabilities of being chosen. An example of implementation of an algorithm $\mathcal{A}$ for the decision tree of figure \ref{fig:sub:Dtree} could be defined as $\mathcal{A}: T \mapsto \mathscr{P}(E)$ with $\mathscr{P}(E)$ being the set of all paths of $T$. For instance, if the objective function is to find the most probable path in the tree, then the solution could be expressed as in (\ref{logpe}).
\begin{align}
\pi^{*} = \argmax\limits_{\pi \in \{\pi_1, \ldots, \pi_6\}} \sum\limits_{e \in \pi} \log \mathbb{P}(e)
\label{logpe}
\end{align}
with $\mathbb{P}(e)=\mathcal{L}_E(e)$ being the probability of edge $e$. Other formulations of (\ref{logpe}) could include for instance the preferences over the edges and define the goal as maximising some expected value.

\section{Translating Tree Problems to Categories}

\subsection{Overview of Category Theory}
\label{overviewCT}

In the following, we give a short introduction to Category Theory and the components relevant to the topic of general problem-solving as previously introduced. For a thorough and in-depth explanation of Category Theory from a mathematical point of view, the reader is advised to use the classical book \cite{mac2013categories}, and to \cite{pierce1991basic,walters1991categories} for the Computer Science point of view.

A category $\mathcal{C}$ is a collection of objects and a collection of arrows called morphisms. It is formally defined as following.

\begin{enumerate} 
\item A class of objects $Ob(\mathcal{C})$. For $X \in Ob(\mathcal{C})$, we can also write $X \in \mathcal{C}$.
\item For every objects $X, Y \in Ob(\mathcal{C})$, the class ${Mor}_{\mathcal{C}}(X,Y)$ defines the class of morphisms from $X$ to $Y$. For $f \in {Mor}_{\mathcal{C}}$, one may also write $f: X \rightarrow Y$. For any objects $X, Y, Z \in Ob(\mathcal{C})$, a composition map $\circ_{X, Y, Z}: {Mor}_{\mathcal{C}}(Y, Z) \times {Mor}_{\mathcal{C}}(X, Y) \rightarrow {Mor}_{\mathcal{C}}(X, Z)$, $(f, g) \mapsto f \circ g$ satisfies:
\begin{enumerate}
	\item Associativity: $(f \circ g) \circ h = f \circ (g \circ h)$
	\item Identity: For each $X \in Ob(\mathcal{C})$, there is a morphism $1_X \in {Mor}_{\mathcal{C}}(X, X)$, called the unit morphism, such that $1_X \circ f = f$ and $g \circ 1_X = g$ for any $f, g$ for which composition holds.
\end{enumerate}
 \end{enumerate}


Another useful category-theoretic construct is the notion of (covariant) functor, which is a morphism of categories. Given two categories $\mathcal{C}$ and $\mathcal{C}'$, a functor $F: \mathcal{C} \rightarrow \mathcal{C}'$ is made of
\begin{enumerate}
	\item A function mapping objects to objects $F: Ob(\mathcal{C}) \rightarrow Ob(\mathcal{C}')$.
	\item For any pair of objects $X, Y \in \mathcal{C}$, we have $F: {Mor}_{\mathcal{C}}(X, Y) \rightarrow {Mor}_{\mathcal{C}'}(F(X), F(Y))$ with the natural requirements of identity and composition:
    \begin{enumerate}
    	\item Identity: $F(1_X) = 1_{F(X)}$
    	\item Composition: $F(f \circ g) = F(f) \circ F(g)$
    \end{enumerate}
\end{enumerate}

Functors will be later used to formalise analogies across problem and solution categories.

\subsection{Problems as categories}
In section \ref{charaMsection}, we have shown that any tree $T_a$ could be encoded as a matrix $M_a$. In theorem \ref{treeCat}, we show that tree problems are in fact a category and we name it $\mathcal{T}$.

\begin{theorem}
Tree problems define a category $\mathcal{T}$.
\label{treeCat}
\end{theorem}
\begin{proof}

In order for $\mathcal{T}$ to be a category, we need to characterise its objects $Ob(\mathcal{T})$, morphisms ${Mor}_{\mathcal{T}}$, and the laws of composition that govern ${Mor}_{\mathcal{T}}$. 

\begin{itemize}

\item Objects: Since each tree is translatable to its characteristic matrix, we will take $Ob(\mathcal{T})$ to be the set of matrices that encode the trees.

\item Morphisms: 
One analytical way of distinguishing between two tree problems is through the existence of a transformation that maps one to the other. These transformations, if they exist, are the morphisms of the category $\mathcal{T}$ that we want to characterise. That is, we need to define the morphisms and their laws of composition, and show that the identity and associativity of morphisms hold. To define ${Mor}_\mathcal{T}$, we define a morphism between two tree matrices $X_{(m, n)}$ and $Y_{(m, n)}$ as the transformation $A_{(n, n)}$ such as $A_{(n, n)} X_{(n, m)}^T = Y_{(n, m)}^T$. Since the number of tree edges usually exceeds the number of features ($n>m$), we need to find the generalised inverse of $X_{(n, m)}^T$ that satisfies (\ref{matA}).
\begin {align}
A_{(n, n)} &= Y_{(n, m)}^T (X_{(n, m)}^T)^{-1} \nonumber  \\
& =Y_{(n, m)}^T X_{(m, n)}^{-1} 
\label{matA}
\end {align} 
To obtain $X^{-1}$, we use the singular value decomposition of $X$ into $P$, $Q$ and $\Delta$, as in (\ref{singvaldecomp}).
\begin {align}
X = P \Delta\ Q^T
\label{singvaldecomp}
\end {align}
where $P$ is an $n \times r$ semiorthogonal matrix, $r$ is the rank of $X$, $\Delta$ is an $r \times r$ diagonal matrix with positive diagonal elements called the singular values of $X$, and $Q$ is an $m \times r$ semiorthogonal matrix. The Moore-Penrose pseudoinverse \cite{grevillegeneralized} of $X$, denoted by $X^{+}$, is the unique $m \times n$ matrix defined by $X^{+} = Q \Delta^{-1} P^T$. The final transformation matrix $A_{(n, n)}$ is therefore computed as (\ref{transfA}).
\begin {align}
A = Y Q \Delta^{-1} P^T
\label{transfA}
\end {align}

The existence of $X^{+}$ and $A$ is guaranteed by the nature of the feature matrices and the fact that $m<\binom{k}{2} + k - 1$, with $k$ being the number of terminal nodes. In the following, we will be using morphisms and matrix transformation interchangeably. After defining the morphisms of $\mathcal{T}$, we prove that the composition laws within $\mathcal{T}$ hold.

\item Composition: Let $f, g \in {Mor}_{\mathcal{T}}$ with $f: T_{\mathcal{P}} \rightarrow T_{\mathcal{P}'}$ and $g: T_{\mathcal{P}'} \rightarrow T_{\mathcal{P}''}$. Given matrices $A_f$ and $A_g$ of $f$ and $g$, and matrices $M_{\mathcal{P}}$, $M_{\mathcal{P}'}$ and $M_{\mathcal{P}''}$ of $T_{\mathcal{P}}$, $T_{\mathcal{P}'}$ and $T_{\mathcal{P}''}$, we have (\ref{compos}).
\begin{subequations}
        \begin{align}
        M_{\mathcal{P}''} &= A_g M_{\mathcal{P}'}  \\
        &= A_g (A_f M_{\mathcal{P}})\\
        &= (A_g A_f )M_{\mathcal{P}}\\
        &= A_{g \circ f} M_{\mathcal{P}}\\
        &= A_h M_{\mathcal{P}}
        \end{align}
    \label{compos}
\end{subequations}
It follows that there exists a morphism $h$ such that $h: T_{\mathcal{P}} \rightarrow T_{\mathcal{P}''}$. Therefore, the composition of morphisms holds and we have (\ref{morphholds}).
\begin{align}
\forall T_\mathcal{P}, T_\mathcal{P}', T_\mathcal{P}'' \in Ob(\mathcal{T})\ \ \ 
{Mor}_\mathcal{T}(T_\mathcal{P}, T_{\mathcal{P}'}) \times {Mor}_\mathcal{T}(T_{\mathcal{P'}}, T_{\mathcal{P}''}) \mapsto {Mor}_\mathcal{T}(T_\mathcal{P}, T_{\mathcal{P}''})
 \label{morphholds}
\end{align} 
The laws of composition need to obey the following.
\begin{enumerate}
\item Associativity: Let $f, g, h \in {Mor}_{\mathcal{T}}$ and their corresponding matrix transformations $A_f$, $A_g$ and $A_h$. Since matrix multiplication is associative $(A_f A_g) A_h=A_f (A_g A_h)$, we have $(f \circ g) \circ h = f \circ (g \circ h)$. 
\item Identity:
Let $f: T_{\mathcal{P}} \rightarrow T_{\mathcal{P}'}$ be a tree morphism and its characteristic matrix transformation $A_f$ that maps $M_{\mathcal{P}}$ to $M_{\mathcal{P}'}$. Let $1_{T_{\mathcal{P}}}: T_{\mathcal{P}} \rightarrow T_{{\mathcal{P}'}}$ be an identity morphism. It must hold that $1_{T_{\mathcal{P}'}} \circ f = f = f \circ 1_{T_\mathcal{P}}$ where $1_{T_\mathcal{P}}$ creates a trivial representation of $T_{\mathcal{P}}$ containing the same structure and features. Similarly, $1_{T_{\mathcal{P}'}}$ creates a trivial representation of $T_{\mathcal{P}'}$. Hence, there exists an identity morphism for all $Obj(\mathcal{T})$. This translates to the existence of identity matrix $A_f$ such as $A_f \times M_{\mathcal{P}} =  M_{\mathcal{P}}$
\end{enumerate}

\end{itemize}

$\mathcal{T}$ is therefore a category and we can illustrate it with the commutative diagram of figure \ref{CommutDiag}.

\end{proof}

\subsection{Solutions as categories}
Similarly to the tree problems category, theorem \ref{solCat} defines the solutions as the category $\mathcal{S}$.

\begin{theorem}
The solutions to tree problems define a category $\mathcal{S}$.
\label{solCat}
\end{theorem}

\begin{proof}
The proof is similar to the proof of theorem \ref{treeCat}. The difference is that the elements of  $Obj(\mathcal{T})$ are $m \times n$ matrices while the elements of $Obj(\mathcal{S})$ are $1 \times n$ matrices since the solutions are binary vectors in $\{0, 1\}^n$.
The commutative diagram of $\mathcal{S}$ is shown in figure \ref{CommutDiagS}.
\end{proof}
\begin{figure}[h]
\centering
\begin{subfigure}{0.4\textwidth}
            \begin{tikzcd}[row sep=2em,column sep=2em]
            & T_{\mathcal{P}} \arrow[ddl,swap,"f"] \\ \\
            T_{\mathcal{P}'} \arrow[rr,swap,"g"] &&
            T_{\mathcal{P}''} \arrow[uul,swap,leftarrow, "g \circ f"]
            \end{tikzcd}
    \caption{Category $\mathcal{T}$}
    \label{CommutDiag}
\end{subfigure}
\begin{subfigure}{0.4\textwidth}
       \begin{tikzcd}[row sep=2em,column sep=2em]
            & S_{\mathcal{P}} \arrow[ddl,swap,"f'"] \\ \\
            S_{\mathcal{P}'} \arrow[rr,swap,"g'"] &&
            S_{\mathcal{P}''} \arrow[uul,swap,leftarrow, "g' \circ f'"]
            \end{tikzcd}
    \caption{Category $\mathcal{S}$}
    \label{CommutDiagS}
        \end{subfigure}
\caption{Commutative diagrams of categories $\mathcal{T}$ and $\mathcal{S}$}
\label{CommutativeDiagrams}
\end{figure}
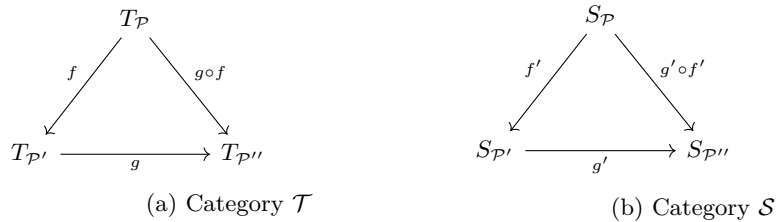
\section{Solving Problems using Functors}
Given problem and solution categories, it is possible to exploit analogies between old and new problems using functors. One could think of an analogy as a structure preserving map from the space of problems to the space of solutions, which rightfully translates to a functor. The analogy <<S' is to S as P' is to P>> can be rewritten as a curried sequence of objects to highlight the transformational aspect: $P \xrightarrow{f} P' \xRightarrow{\mathcal{F}} S \xrightarrow{\mathcal{F} f} S'$. If we know $P$, $S$, and $P'$, and wish to learn about $S'$, we could learn the functor $\mathcal{F}: P' \rightarrow S'$. Using the knowledge about $P'$, how it relates to $P$, and the structure of $\mathcal{F}$, we can either use $\mathcal{F}(P')$ to further learn $S'$ and how it relates to $S$ or use $\mathcal{F}f$ to infer $S'$ from $S$. The solution could be found in different ways and with different complexities \cite{izbicki2013two}, depending on how we traverse the commutative diagram. In the following, we propose to characterise the functor that maps category $\mathcal{T}$ to category $\mathcal{S}$.

\subsection{Existence of functors between problems and solutions}

Whenever we have a collection of problems, we want to be able to know how to relate them. Mapping problems to solutions requires a level of identification between the two. An isomorphism for instance is the type of strong identification between two categories. If two categories are isomorphic, then they are the same and perhaps differ only in notation. However, isomorphisms are in general rare and difficult to characterise. We can instead ``weaken'' the isomorphism by descending from isomorphism of categories to equivalence of categories, and eventually to adjunction of functors between categories \cite{phillips2017general}. This weakening holds in particular for our case of problems since some problems might not have solutions and vice versa. The concept of equivalence of categories is used to identify categories since it is weaker and more generic. We define it by a functor $F: \mathcal{T} \rightarrow \mathcal{S}$ which is an isomorphism of categories up to isomorphisms. If $F$ is an equivalence of categories, then it induces a bijection between the classes of isomorphic objects of $\mathcal{T}$ and $\mathcal{S}$ even if $F$ is not bijective on all the objects \cite{mac2013categories}. Thus the bijection $\mathcal{F}$ is defined as $\mathcal{F}: (\mathcal{T}/\simeq) \mapsto (\mathcal{S}/\simeq)$ and could mainly serve the purpose of identifying classes of problems and solutions as opposed to a one-to-one identification of the components of problems and solutions.

In theorem \ref{functorproof}, we prove the existence of the functor $F: \mathcal{T} \rightarrow \mathcal{S}$ using the previously constructed categories. We will later propose the weaker version of $F$ in terms of equivalences and through a metric on tree problems.

\begin{theorem}
There exists a functor $F$ from the category of tree problems $\mathcal{T}$ to the category of solutions $\mathcal{S}$.
\label{functorproof}
\end{theorem}

\begin{proof}

For $F$ to be a functor from $\mathcal{T}$ to $\mathcal{S}$, we must show that $F$ preserves identity morphisms and composition of morphisms as introduced in section \ref{overviewCT}.

\begin{enumerate}
    \item Identity: Let $T_{\mathcal{P}} \in \mathcal{T}$ be given and let $1_{T_\mathcal{P}}$ be the identity morphism in $\mathcal{T}$ corresponding to $T_{\mathcal{P}}$. Let $1_{F({\mathcal{P}})}$ be the identity morphism in $\mathcal{S}$ corresponding to $F(T_{\mathcal{P}})$. We need to show that $F(1_{T_\mathcal{P}})=1_{F({T_\mathcal{P}})}$. In the category $\mathcal{T}$, the identity morphism $1_{T_\mathcal{P}}$ creates a trivial tree problem from an existing one. Similarly, in $\mathcal{S}$, the identity morphism $1_{F({T_\mathcal{P}})}$ also creates a trivial structure from the same solution. The functor $F$ maps the morphism $1_{T_\mathcal{P}}: {T_\mathcal{P}} \mapsto {T_\mathcal{P}}$ in $\mathcal{T}$ to $F(1_{T_\mathcal{P}}): F({T_\mathcal{P}})  \mapsto F({T_\mathcal{P}})$ in $\mathcal{S}$.
    
    Therefore, $F(1_{T_\mathcal{P}}) = 1_{F({T_\mathcal{P}})}$ and the functor $F$ preserves identity morphisms.
  
        \item Composition: Let $f, g \in {Mor}_{\mathcal{T}}$ such that $f: T_{\mathcal{P}} \mapsto T_{\mathcal{P}'}$ and $g: T_{\mathcal{P}'} \mapsto T_{\mathcal{P}''}$. Let also $F(f), F(g) \in {Mor}_{\mathcal{S}}$ be such that $F(f): F(T_{\mathcal{P}}) \mapsto T_{\mathcal{P}'}$ and $F(g): F(T_{\mathcal{P}'}) \mapsto T_{\mathcal{P}''}$. We need to show that $F(g \circ f) = F(g) \circ F(f)$. We have $F(g \circ f) = F(g(f(T_{\mathcal{P}}))) = F(g(T_{\mathcal{P}'})) = F(T_{\mathcal{P}''})$ and    
    $F(g) \circ F(f) = F(g(F(f(T_{\mathcal{P}})))) = F(g(F(T_{\mathcal{P}'}))) = F(T_{\mathcal{P}''})$.
    
    Hence $F$ preserves the composition of morphisms.
\end{enumerate}

\end{proof}

$F$ is therefore a functor from $\mathcal{T}$ to $\mathcal{S}$ and has the commutative diagram of figure \ref{functorCommDiag}.
\begin{figure}
    \centering
            \begin{tikzcd}[row sep=1.5em,column sep=3em]
            &&&& S_{\mathcal{P}} \arrow[ddl,swap,"F (f)"] \arrow[ddr,"F(g \circ f)"] \\
            & T_{\mathcal{P}} \arrow[ddl,swap,"f"] \arrow[urrr,dashed,"F", style={pos=0.7}] \\
            &&& S_{\mathcal{P}'} \arrow[rr,"F (g)"] &&
                S_{\mathcal{P}''}\\
T_{\mathcal{P}'} \arrow[urrr,dashed,"F", style={pos=0.8}] \arrow[rr,swap,"g"] &&
T_{\mathcal{P}''} \arrow[urrr,dashed, "F", style={pos=0.4}] \arrow[uul,leftarrow,swap,crossing over, "g \circ f", style={pos=0.6}]
            \end{tikzcd}
    \caption{Commutative diagram of the functor $F$}    
    \label{functorCommDiag}
\end{figure}
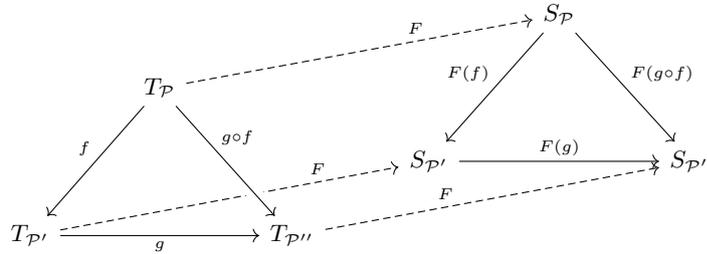

\subsection {From equivalence to metric}

The equivalences of categories of trees $(\mathcal{T}/\simeq)$ define what can be identified as the level of similarities or analogy between the problems that they represent. Similarly, the equivalence of categories of solutions $(\mathcal{S}/\simeq)$ defines the levels of similarities between solutions. If the tree $T_{\mathcal{P}} \in Ob(\mathcal{T})$ is analogous to other trees $\{T_{\mathcal{P}'}\}_{{\mathcal{P}'} \neq {\mathcal{P}}}$, it would be useful to find the ``most'' analogous ones, for instance to transfer knowledge between the closest ones \cite{taylor2009transfer}. This could be done by defining a distance that measures how analogous they are: the more analogous $T_{\mathcal{P}}$ and $T_{\mathcal{P}'}$ are, the smaller $d(T_{\mathcal{P}}, T_{\mathcal{P}'})$ should be.
We propose to construct such a distance on $Ob(\mathcal{T})$ and $Ob(\mathcal{S})$ to identify the objects more or less similar. Recall that by definition, a binary relation $\simeq$ is an equivalence relation if and only if satisfies reflexivity, symmetry and transitivity. 
These conditions are satisfied by the equality relation $=$ and are ``natural'' to express what a notion of analogy should satisfy. In that way, equality can be viewed as a particular
case of analogy. On the other hand, analogies, as formalised by the concept of equivalence relations, can be viewed as generalisation of equality.

\subsection{Problem and solution metrics}
	
A metric is the mathematical notion of distance that give structure and shape to a set of objects by forming a space. A function $d(T_{\mathcal{P}}, T_{\mathcal{P}'})$ is a tree problem metric if, for all $T_{\mathcal{P}}, T_{\mathcal{P}} \in Ob(\mathcal{T})$:
\begin {enumerate}
    \item Distances are non-negative: $d(T_{\mathcal{P}}, T_{\mathcal{P}'}) \ge 0$ 
    \item Distance is equal to zero when trees are identical: $d(T_{\mathcal{P}}, T_{\mathcal{P}'}) = 0   \iff T_{\mathcal{P}} = T_{\mathcal{P}'}$
    \item Distance is symmetric: $d(T_{\mathcal{P}}, T_{\mathcal{P}'}) = d(T_{\mathcal{P}'}, T_{\mathcal{P}})$
    \item Distances satisfy the triangle inequality: $\forall T_{\mathcal{P}''} \in Ob(\mathcal{T})$, $d(T_{\mathcal{P}}, T_{\mathcal{P}'}) \leq d(T_{\mathcal{P}}, T_{\mathcal{P}''}) + d(T_{\mathcal{P}''}, T_{\mathcal{P}'})$
\end {enumerate}

Now we can define the tree metric based on the characteristic function (\ref{cvx}).
   
\begin{theorem}
The function $d_\lambda: Ob(\mathcal{T}) \times Ob(\mathcal{T}) \mapsto \mathbb{R}$ given by $d_\lambda(T_{\mathcal{P}}, T_{\mathcal{P}'}) = \lVert \phi_\lambda(T_{\mathcal{P}}) - \phi_\lambda(T_{\mathcal{P}'}) \rVert$ is a metric on $Ob(\mathcal{T})$, with $\lVert . \rVert$ being the Euclidean distance and $\lambda \in [0,1]^{m+1}$.
\end{theorem}

\begin{proof}
The proof is very similar to the one for phylogenetic trees \cite{kendall2015tree}.
\end{proof}

In a similar way, we could prove that there exists a metric on the solution space. This case is more trivial since the solutions are binary vectors of $\{0, 1\}^n$.

\begin{theorem}
The function $d: Ob(\mathcal{S}) \times Ob(\mathcal{S}) \mapsto \mathbb{R}$ given by $d(S_{\mathcal{P}}, S_{\mathcal{P}'}) = \lVert S_{\mathcal{P}} - S_{\mathcal{P}'} \rVert$ is a metric on $Ob(\mathcal{S})$, with $S_{\mathcal{P}}, S_{\mathcal{P}'} \in \{0, 1\}^n$ and $\lVert . \rVert$ being the Euclidean distance.
\end{theorem}

The metrics can be used to measure how problems and solutions are relatable. This way of characterising the existence of functors allows us to find the most analogous known problem(s) to a given situation. Given a target problem $P$ we could find the set $\{(P', S')\}_{P' \simeq P}$ of equivalent problems that were previously solved, find the convex transformation $f$ that maps $P'$ to $P$ and compute $S$ as $\mathcal{F}f(S')$. This transformation is the type of transfer of knowledge from past to new situations.
 
\section{Conclusions}

The paper proposes a category-theoretic approach that formalises problems that are represented as tree-like structures. The existence of equivalence relationships across the categories of problems and their corresponding categories of solutions is established using functors. Implementing the functors corresponds therefore to solving the problems through means of analogy.

The proposed formalism has yet to be tested on concrete instances of tree-like problems such as maze problems. The future direction is to characterise the functors as encoders in a way similar to \cite{reed2015deep} and learn the generalised solutions to different maze problem.

\bibliographystyle{splncs04}
\bibliography{ref}

\begin{thebibliography}{10}
\providecommand{\url}[1]{\texttt{#1}}
\providecommand{\urlprefix}{URL }
\providecommand{\doi}[1]{https://doi.org/#1}

\bibitem{arjonilla2017general}
Arjonilla, F.J., Ogata, T.: General problem solving with category theory. arXiv
  preprint arXiv:1709.04825  (2017)

\bibitem{arjonilla2015three}
Arjonilla~Garc{\'\i}a, F.: A three-component cognitive theory. Master's thesis
  (2015)

\bibitem{arzi1999perceive}
Arzi-Gonczarowski, Z.: Perceive this as that--analogies, artificial perception,
  and category theory. Annals of Mathematics and Artificial Intelligence
  \textbf{26}(1-4),  215--252 (1999)

\bibitem{bonet2009solving}
Bonet, B., Geffner, H.: Solving pomdps: Rtdp-bel vs. point-based algorithms.
  pp. 641--1646 (2009)

\bibitem{bonet2015policies}
Bonet, B., Geffner, H.: Policies that generalize: Solving many planning
  problems with the same policy. vol.~15 (2015)

\bibitem{cardona2013cophenetic}
Cardona, G., Mir, A., Rossell{\'o}, F., Rotger, L., S{\'a}nchez, D.: Cophenetic
  metrics for phylogenetic trees, after sokal and rohlf. BMC bioinformatics
  \textbf{14}(1), ~3 (2013)

\bibitem{csikvari2013graph}
Csikv{\'a}ri, P., Lin, Z.: Graph homomorphisms between trees. arXiv preprint
  arXiv:1307.6721  (2013)

\bibitem{diuk2013divide}
Diuk, C., Schapiro, A., C{\'o}rdova, N., Ribas-Fernandes, J., Niv, Y.,
  Botvinick, M.: Divide and conquer: hierarchical reinforcement learning and
  task decomposition in humans. In: Computational and robotic models of the
  hierarchical organization of behavior, pp. 271--291. Springer (2013)

\bibitem{eilenberg1945general}
Eilenberg, S., MacLane, S.: General theory of natural equivalences.
  Transactions of the American Mathematical Society  \textbf{58},  231--294
  (1945)

\bibitem{fodor1988connectionism}
Fodor, J.A., Pylyshyn, Z.W.: Connectionism and cognitive architecture: A
  critical analysis. Cognition  \textbf{28}(1-2),  3--71 (1988)

\bibitem{gordon1989greedoid}
Gordon, G., McMahon, E.: A greedoid polynomial which distinguishes rooted
  arborescences. Proceedings of the American Mathematical Society
  \textbf{107}(2),  287--298 (1989)

\bibitem{grevillegeneralized}
Greville, T.N.: Generalized inverses: Theory and applications adi ben-israel

\bibitem{healy2000category}
Healy, M.J.: Category theory applied to neural modeling and graphical
  representations. In: Neural Networks, 2000. IJCNN 2000, Proceedings of the
  IEEE-INNS-ENNS International Joint Conference on. vol.~3, pp. 35--40. IEEE
  (2000)

\bibitem{hutter2004universal}
Hutter, M.: Universal artificial intelligence: Sequential decisions based on
  algorithmic probability. Springer Science \& Business Media (2004)

\bibitem{izbicki2013algebraic}
Izbicki, M.: Algebraic classifiers: a generic approach to fast
  cross-validation, online training, and parallel training. In: International
  Conference on Machine Learning. pp. 648--656 (2013)

\bibitem{izbicki2013two}
Izbicki, M.: Two monoids for approximating np-complete problems  (2013)

\bibitem{keisler1990model}
Keisler, H.J., Chang, C.C.: Model theory. North-Holland Amsterdam (1990)

\bibitem{kendall2015tree}
Kendall, M., Colijn, C.: A tree metric using structure and length to capture
  distinct phylogenetic signals. arXiv preprint arXiv:1507.05211  (2015)

\bibitem{laird2012soar}
Laird, J.E.: The Soar cognitive architecture (2012)

\bibitem{laird2010cognitive}
Laird, J.E., Wray~III, R.E.: Cognitive architecture requirements for achieving
  agi. In: Proc. of the Third Conference on Artificial General Intelligence.
  pp. 79--84 (2010)

\bibitem{liao2017visual}
Liao, J., Yao, Y., Yuan, L., Hua, G., Kang, S.B.: Visual attribute transfer
  through deep image analogy. arXiv preprint arXiv:1705.01088  (2017)

\bibitem{mac2013categories}
Mac~Lane, S.: Categories for the working mathematician, vol.~5. Springer
  Science \& Business Media (2013)

\bibitem{magnan1994category}
Magnan, F., Reyes, G.E.: Category theory as a conceptual tool in the study of
  cognition. The logical foundations of cognition pp. 57--90 (1994)

\bibitem{Maze2Tree}
MAZEMASTERS: Maze to tree [online]  (2007),
  \url{https://www.youtube.com/watch?v=k1tSK5V1pds}

\bibitem{newell1992unified}
Newell, A.: Unified theories of cognition and the role of soar. In: SOAR: A
  cognitive architecture in perspective, pp. 25--79. Springer (1992)

\bibitem{newell1959report}
Newell, A., Shaw, J.C., Simon, H.A.: Report on a general problem solving
  program. In: IFIP congress. vol.~256, p.~64 (1959)

\bibitem{phillips2017general}
Phillips, S.: A general (category theory) principle for general intelligence:
  Duality (adjointness). In: International Conference on Artificial General
  Intelligence. pp. 57--66. Springer (2017)

\bibitem{phillips2010categorial}
Phillips, S., Wilson, W.H.: Categorial compositionality: A category theory
  explanation for the systematicity of human cognition. PLoS computational
  biology  \textbf{6}(7),  e1000858 (2010)

\bibitem{phillips2016systematicity}
Phillips, S., Wilson, W.H.: Systematicity and a categorical theory of cognitive
  architecture: universal construction in context. Frontiers in psychology
  \textbf{7}, ~1139 (2016)

\bibitem{pierce1991basic}
Pierce, B.C.: Basic category theory for computer scientists. MIT press (1991)

\bibitem{ramirez2010new}
Ram{\i}rez, J.D.G.: A new foundation for representation in cognitive and brain
  science: Category theory and the hippocampus. Departamento de Automatica,
  Ingenier{\i}a Electronica e Informatica Industrial Escuela Tecnica Superior
  de Ingenieros Industriales  (2010)

\bibitem{rasmussen2014hierarchical}
Rasmussen, D.: Hierarchical reinforcement learning in a biologically plausible
  neural architecture  (2014)

\bibitem{reed2015deep}
Reed, S.E., Zhang, Y., Zhang, Y., Lee, H.: Deep visual analogy-making. In:
  Advances in neural information processing systems. pp. 1252--1260 (2015)

\bibitem{rosa2016framework}
Rosa, M., Feyereisl, J., Collective, T.G.: A framework for searching for
  general artificial intelligence. arXiv preprint arXiv:1611.00685  (2016)

\bibitem{sharma2007transfer}
Sharma, M., Holmes, M.P., Santamar{\'\i}a, J.C., Irani, A., Ram, A.: Transfer
  learning in real-time strategy games using hybrid cbr/rl.

\bibitem{taylor2009transfer}
Taylor, M.E., Stone, P.: Transfer learning for reinforcement learning domains:
  A survey. Journal of Machine Learning Research  \textbf{10}(Jul),  1633--1685
  (2009)

\bibitem{tsuchiya2016using}
Tsuchiya, N., Taguchi, S., Saigo, H.: Using category theory to assess the
  relationship between consciousness and integrated information theory.
  Neuroscience research  \textbf{107}, ~1--7 (2016)

\bibitem{veness2010monte}
Veness, J., Ng, K.S., Hutter, M., Uther, W., Silver, D.: A monte-carlo aixi
  approximation  (2010)

\bibitem{walters1991categories}
Walters, R.F.C.: Categories and computer science. Cambridge University Press
  (1991)

\end{thebibliography}

\end{document}